\documentclass[11pt]{article}

\usepackage{amsmath, amsfonts, amssymb, amsthm,  graphicx, mathtools, enumerate}

\usepackage[blocks, affil-it]{authblk}

\usepackage[numbers, square]{natbib}
\usepackage[CJKbookmarks=true,
            bookmarksnumbered=true,
			bookmarksopen=true,
			colorlinks=true,
			citecolor=red,
			linkcolor=blue,
			anchorcolor=red,
			urlcolor=blue]{hyperref}
\usepackage[usenames]{color}

\usepackage[letterpaper, left=1.2truein, right=1.2truein, top = 1.2truein, bottom = 1.2truein]{geometry}
\usepackage[ruled, vlined, lined, commentsnumbered]{algorithm2e}

\usepackage{prettyref,soul}

\newtheorem{lemma}{Lemma}[section]

\newtheorem{thm}{Theorem}[section]

\newtheorem{corollary}{Corollary}[section]

\newrefformat{eq}{(\ref{#1})}
\newrefformat{chap}{Chapter~\ref{#1}}
\newrefformat{sec}{Section~\ref{#1}}
\newrefformat{algo}{Algorithm~\ref{#1}}
\newrefformat{fig}{Fig.~\ref{#1}}
\newrefformat{tab}{Table~\ref{#1}}
\newrefformat{rmk}{Remark~\ref{#1}}
\newrefformat{clm}{Claim~\ref{#1}}
\newrefformat{def}{Definition~\ref{#1}}
\newrefformat{cor}{Corollary~\ref{#1}}
\newrefformat{lmm}{Lemma~\ref{#1}}
\newrefformat{lemma}{Lemma~\ref{#1}}
\newrefformat{prop}{Proposition~\ref{#1}}
\newrefformat{app}{Appendix~\ref{#1}}
\newrefformat{ex}{Example~\ref{#1}}
\newrefformat{exer}{Exercise~\ref{#1}}
\newrefformat{soln}{Solution~\ref{#1}}
\newrefformat{cond}{Condition~\ref{#1}}

\def\text#1{\mbox{\rm #1}}

\newcommand{\argmin}{\mathop{\rm argmin}}

\title{Exact Exponent in Optimal Rates for Crowdsourcing
}
\author[1]{Chao Gao}
\author[1]{Yu Lu}
\author[2]{Dengyong Zhou}
\affil[1]{
Yale University
}
\affil[2]{
Microsoft Research, Redmond
}

\begin{document}
\maketitle

\begin{abstract}
In many machine learning applications, crowdsourcing has become the primary  means for label collection. In this paper, we study the optimal error rate for aggregating labels provided by a set of non-expert workers. Under the classic Dawid-Skene model, we establish matching upper and lower bounds with an exact exponent $mI(\pi)$ in which  $m$ is the number of workers and $I(\pi)$ the average Chernoff information that characterizes the workers' collective ability. Such an exact characterization of the error exponent allows us to state a precise sample size requirement $m>\frac{1}{I(\pi)}\log\frac{1}{\epsilon}$ in order to achieve an $\epsilon$ misclassification error. In addition,  our results imply the optimality of various EM algorithms for crowdsourcing initialized by consistent estimators.
\smallskip
\end{abstract}

\section{Introduction}

In many machine learning problems such as image classification and speech recognition, we need a large amount of labeled data.
Crowdsourcing provides an efficient while inexpensive way to collect labels. On a commercial crowdsourcing platform like Amazon Mechanical Turk \citep{mturk},   in general, it takes only few hours to obtain hundreds of thousands labels from crowdsourcing workers worldwide, and each label costs only several cents.

Though massive in amount, the crowdsourced labels are usually fairly noisy. The low quality is partially due to the lack of domain expertise from the workers and  presence of spammers. To overcome this issue, a common strategy is to repeatedly label each item by different workers, and then estimate  truth from the redundant labels, for example, using majority voting. Since  the pioneering work by Dawid and Skene \citep{DawSke79},  which jointly estimates truth and workers' abilities via a simple EM algorithm, various approaches have been developed in recent years for aggregating noisy crowdsourced labels. See \cite{WRWBM09,welinder2010nips, RayYuZha10,ghosh2011moderates, bachrach2012icml,LiuPenIhl12,zhoplaby12,DalDasKum2013,zhou2014aggregating, venanzi2014community,parisi2014ranking,tianzhu2015nips} and references therein.

Compared with the active progress in aggregation algorithms, statistical understandings of crowdsourcing do not get much attention except \cite{gao2013minimax, karger14,zhang2014spectral, berend2015finite}. These papers not only show exponential convergence rates for several estimators, they also provide lower bounds to justify the optimality of the rates. However, the exponents found in these work are not matched in their upper and lower bounds. They are optimal only up to some unspecified constants. The main focus of this paper is to find the \emph{exact} error exponent to better guide algorithm design and optimization.

\paragraph{Main Contribution.} We study the minimax rate of misclassification for estimating the truth from crowdsourced labels. We provide upper and lower bounds with exact exponents that match each other. The exponent has a natural interpretation of the collective wisdom of a crowd. In the special case where each worker's ability is modeled by a real number $p_i\in[0,1]$, the exponent takes a simple form $-(1+o(1))mI(p)$ with $I(p)=-\frac{1}{m}\sum_{i=1}^m\log\left(2\sqrt{p_i(1-p_i)}\right)$ being the average R\'{e}nyi divergence of order $1/2$. Therefore, in order to achieve an error of $\epsilon$ in the misclassification proportion, it is necessary and sufficient that the number of workers $m$ satisfies $m\geq (1+o(1))I(p)^{-1}\log(1/\epsilon)$. Note that in previous work, only $m = \Omega\left(I(p)^{-1}\log(1/\epsilon)\right)$ can be claimed. Moreover, our general theorem has implications on the convergence rates of several existing algorithms. \\

This paper is organized as follows. In Section \ref{sec:setting}, we present the problem setting. In Section \ref{sec:main}, given the workers' abilities,  we derive the optimal error exponent. In Section \ref{sec:adapt}, we show  that spectral methods can be used to achieve the optimal error exponent, followed by a  discuss on other algorithms in Section \ref{sec:disc}. The proofs are gathered in Section \ref{sec:proof}.

\section{Problem Setting} \label{sec:setting}

Let us start from the classic model proposed by Dawid and Skene \citep{DawSke79}. Assume there are $m$ workers and $n$ items to label. Denote the true label of the $j$th item by $y_j$ that takes on a value in $[k]=\{1,2,...,k\}$. Let $X_{ij}$ be the label given by the $i$th worker to the $j$th item. The  ability of the $i$th worker is assumed to be fully characterized by a confusion matrix
\begin{equation}
\pi_{gh}^{(i)}=\mathbb{P}(X_{ij}=h|y_j=g), \label{eq:prob}
\end{equation}
which satisfies the probabilistic constraint $\sum_{h=1}^k\pi_{gh}^{(i)}=1$. Given $y_j=g$, $X_{ij}$ is generated by a multinomial distribution with parameter $\pi_{g*}^{(i)}=\left(\pi_{g1}^{(i)},...,\pi_{gk}^{(i)}\right)$. Our goal is to estimate the true labels $y=(y_1, \cdots, y_n)$ using the observed labels $\{X_{ij}\}$. Denote the estimate by $\hat{y}=(\hat{y}_1,...,\hat{y}_n).$   The loss is measured by  the error rate
\begin{equation}
L(\hat{y}, y) = \frac{1}{n}\sum_{j=1}^n\mathbb{I}\{\hat{y}_j\neq y_j\}. \label{eq:loss}
\end{equation}
We would like to remark that the true labels are considered as deterministic here.  It is straightforward to generalize our results to stochastic labels  generated from a distribution.  Also, we assume that every worker has labeled every item. Otherwise, we can regard the missing labels as a new category and the results in this paper stay the same.

\section{Main Results} \label{sec:main}
In this section, we assume the confusion matrices $\{\pi^{(i)}\}$ are known. Our goal is to establish the optimal error rate with respect to the loss in Equation (\ref{eq:loss}).  Let $\mathbb{P}_{\pi,y}$ be the joint probability distribution of the data $\{X_{ij}\}$ given $\pi$ and $y$ specified in (\ref{eq:prob}), and let $\mathbb{E}_{\pi,y}$ be the associated expectation operator. Then the optimality is characterized by
\begin{equation}
\mathcal{M}=\inf_{\hat{y}}\sup_{y\in[k]^n}\mathbb{E}_{\pi,y}L(\hat{y}, y),\label{eq:minimax}
\end{equation}
which identifies the lowest error rate that we can achieve uniformly over all possible true labels.

Our main result of the paper is to show that under some mild condition the minimax risk (\ref{eq:minimax}) converges to zero exponentially fast with an exponent that characterizes the collective wisdom of a crowd. Specifically, the error exponent is $-mI(\pi)$ with
\begin{equation}
I(\pi)=\min_{g\neq h}C(\pi_{g*},\pi_{h*}), \label{eq:exponent}
\end{equation}
where $C(\pi_{g*},\pi_{h*})$ is given as
\[-\min_{0\leq t\leq 1}\frac{1}{m}\sum_{i=1}^m\log\left(\sum_{l=1}^k\left(\pi_{gl}^{(i)}\right)^{1-t}\left(\pi_{hl}^{(i)}\right)^t\right).\]

To better present our main result, let us introduce some notations. Let $\rho_m = \min_{i,g,l} \pi_{gl}^{(i)}$. Suppose the minimum of $C(\pi_{g*},\pi_{h*})$ is achieved at $g=a$ and $h=b$. For any $\alpha>0$, we define a set of workers
$$\mathcal{A}_\alpha = \left\{ i \in [m]: \pi_{aa}^{(i)} \ge (1+\alpha) \pi_{ab}^{(i)}, \pi_{bb}^{(i)} \ge (1+\alpha) \pi_{ba}^{(i)}   \right\}.$$
These workers in $\mathcal{A}_\alpha$ have better expertise in distinguishing between categories $a$ and $b$.  Then,  our main result can be summarized into the following theorem.
\begin{thm}\label{thm:main}
Assume $\log k=o(mI(\pi))$, $| \log \rho_m | = o( \rho_m |\mathcal{A}_{0.01}|^{1/2})$ and $| \log \rho_m | = o(\sqrt{m} I(\pi))$, as $m \to \infty$. Then, we have
\begin{eqnarray*}
\inf_{\hat{y}}\sup_{y\in[k]^n}\mathbb{E}_{\pi,y}L(\hat{y}, y)  = \exp\left(-(1+o(1))mI(\pi)\right),
\end{eqnarray*}
where $I(\pi)$ is defined by (\ref{eq:exponent}).
\end{thm}

In Theorem \ref{thm:main}, the assumption that $| \log \rho_m | = o(\rho_m |\mathcal{A}_{0.01}|^{1/2} )$ can be relaxed to that $| \log \rho_m | = o(\rho_m \alpha |\mathcal{A}_{\alpha}|^{1/2})$ for some $\alpha>0$. To better present our result, we set $\alpha=0.01$ in the theorem. To prove the upper bound, we only need the first assumption $\log k=o(mI(\pi))$. The other two assumptions on $\rho_m$ are used for proving the lower bound. One could imagine that the larger $\rho_m$ is, the more mistake we might make to estimate the true labels. When there is a constant $c$ (independent of $m$) such that $\rho_m \ge c$,  the last two assumptions reduce to $|\mathcal{A}_{0.01}| \to \infty$ and $\sqrt{m} I(\pi) \to \infty$. That means as long as $I(\pi) = \Omega(1/\sqrt{m})$ and the number of experts goes to infinity as $m$ grows, $\exp(-(1+o(1))mI(\pi))$ serves as a valid lower bound.


Theorem \ref{thm:main} characterizes the optimal error rate for estimating the ground truth with crowdsourced labels. It implies $\exp\left(-(1+o(1))mI(\pi)\right)$ is the best error rate that can be achieved by any algorithm. Moreover, it also implies there exists an algorithm that can achieve this optimal rate. The error exponent depends on an important quantity $I(\pi)$. When $m=1$ and $k=2$, this theorem reduces to the Chernoff-Stein Lemma \cite{cover12}, in which $I(\pi)$ is the Chernoff information between probability distributions. For the general problem, $C(\pi_{g*},\pi_{h*})$ can be understood as the
average Chernoff information between $\{\pi_{g*}^{(i)}\}_{i=1}^m$ and $\{\pi_{h*}^{(i)}\}_{i=1}^m$, which measures the collective ability of the $m$ workers to distinguish between items with label $g$ and items with label $h$. Then, $I(\pi)$ is the collective ability of the $m$ workers to distinguish between any two items of different labels. The higher the overall collective ability $mI(\pi)$, the smaller the optimal rate.

By Markov's inequality, Theorem \ref{thm:main} implies
$$\frac{1}{n}\sum_{j=1}^n\mathbb{I}\{\hat{y}_j\neq y_j\}\leq \exp\left(-(1+o(1))mI(\pi)\right),$$
with probability tending to $1$. This allows a precise statement for a sample size requirement to achieve a prescribed error. If it is required that the misclassification proportion is no greater than $\epsilon$, then the number of workers should satisfy $m\geq (1+o(1))\frac{1}{I(\pi)}\log\frac{1}{\epsilon}$. A special case is $\epsilon< n^{-1}$. Since $\frac{1}{n}\sum_{j=1}^n\mathbb{I}\{\hat{y}_j\neq y_j\}$ only takes value in $\{0,n^{-1},2n^{-1},...,1\}$, an error rate smaller than $n^{-1}$ implies that every item is correctly labeled. Therefore, as long as $m> (1+o(1))\frac{1}{I(\pi)}\log n$, the misclassification rate is $0$ with high probability.

When $k=2$, a special case of the general Dawid-Skene model takes the simple form
\begin{equation}
\begin{bmatrix}
\pi_{11}^{(i)} & \pi_{12}^{(i)} \\
\pi_{21}^{(i)} & \pi_{22}^{(i)}
\end{bmatrix}=\begin{bmatrix}
p_i & 1-p_i \\
1-p_i & p_i
\end{bmatrix}.\label{eq:special}
\end{equation}
This is referred to as the one-coin model, because the ability of each worker is parametrized by a biased coin with bias $p_i$. In this special case, $I(\pi)$ takes the following simple form
\begin{equation}
I(\pi)=I(p)=-\frac{1}{m}\sum_{i=1}^m\log\left(2\sqrt{p_i(1-p_i)}\right).\label{eq:one-coin}
\end{equation}
Note that $-2\log\left(2\sqrt{p_i(1-p_i)}\right)$ is the R\'{e}nyi divergence of order $1/2$ between $\text{Bernoulli}(p_i)$ and $\text{Bernoulli}(1-p_i)$. Let us summarize the optimal convergence rate for the one-coin model in the following corollary.

\begin{corollary}\label{cor:one-coin}
Assume $\max_{1\leq i\leq m}(|\log(p_i)|\vee|\log(1-p_i)|)=o(mI(p))$, Then, we have
\begin{eqnarray*}
&& \inf_{\hat{y}}\sup_{y\in\{1,2\}^n}\mathbb{E}_{p,y}L(\hat{y}, y) = \exp\left(-(1+o(1))mI(p)\right),
\end{eqnarray*}
where $I(p)$ is defined by (\ref{eq:one-coin}).
\end{corollary}

Corollary \ref{cor:one-coin} has a weaker assumption than that of Theorem \ref{thm:main}. When each $p_i$ is assumed to be in the interval $[c,1-c]$ with some constant $c\in (0,1/2)$, the assumption of Corollary \ref{cor:one-coin} reduces to $mI(p)\rightarrow\infty$, which is actually the necessary and sufficient condition for consistency. The result of Corollary \ref{cor:one-coin} is very intuitive. Note that the R\'{e}nyi divergence $-2\log\left(2\sqrt{p_i(1-p_i)}\right)$ is decreasing for $p_i\in[0,1/2]$ and increasing for $p_i\in[1/2,1]$. When most workers have $p_i$'s that are close to $1/2$, then the rate of convergence will be slow. On the other hand, when $p_i$ is either close to $0$ or close to $1$, that worker has a high ability, which will contribute to a smaller convergence rate. It is interesting to note that the result is symmetric around $p_i=1/2$. This means for adversarial workers with $p_i<1/2$, an optimal algorithm can invert their labels and still get useful information.

\section{Adaptive Estimation} \label{sec:adapt}

The optimal rate in Theorem \ref{thm:main} can be achieved by the following procedure:
\begin{equation}
\hat{y}_j=\arg\max_{g\in[k]}\prod_{i\in[m]}\prod_{h\in[k]}\left(\pi_{gh}^{(i)}\right)^{\mathbb{I}\{X_{ij}=h\}}.\label{eq:MLE}
\end{equation}
This is the maximum likelihood estimator. When $k=2$, it reduces to the likelihood ratio test by Neyman and Pearson \citep{neyman1933problem}. However, (\ref{eq:MLE}) is not practical because it requires the knowledge of the confusion matrix $\pi^{(i)}$ for each $i\in[m]$. A natural data-driven alternative is to first get an accurate estimator $\hat{\pi}$ of $\pi$ in (\ref{eq:MLE}) and then consider the plug-in estimator,
\begin{equation}
\hat{y}_j=\arg\max_{g\in[k]}\prod_{i\in[m]}\prod_{h\in[k]}\left(\hat{\pi}_{gh}^{(i)}\right)^{\mathbb{I}\{X_{ij}=h\}}.\label{eq:est}
\end{equation}
In the next theorem, we show that as long as $\hat{\pi}$ is sufficiently accurate, (\ref{eq:est}) will also achieve the optimal rate in Theorem \ref{thm:main}.

\begin{thm}\label{thm:adaptive}
Assume that, as $m \to \infty$,
\begin{equation} \label{eq:condition}
\mathbb{P}\left(\max_{g\in[k]}\sum_{i\in[m]}\max_{h\in[k]}\left|\log\hat{\pi}_{gh}^{(i)}-\log\pi_{gh}^{(i)}\right|> \delta\right)\rightarrow 0
\end{equation}
with $\delta$ such that $\delta+\log k=o(mI(\pi))$. Then, for any $y\in[k]^n$, we have
$$\frac{1}{n}\sum_{j=1}^n\mathbb{I}\{\hat{y}_j\neq y_j\}\leq \exp\left(-(1+o(1))mI(\pi)\right),$$
with probability tending to $1$, where $I(\pi)$ is defined by (\ref{eq:exponent}).
\end{thm}

Theorem \ref{thm:adaptive} guarantees that as long as the confusion matrices can be consistently estimated, the plugged-in MLE (\ref{eq:est}) achieves the optimal error rate. In what follows, we apply this result to verify the optimality of some methods proposed in the literature.

\subsection{Spectral Methods}
Let us first look at the spectral method proposed in \cite{zhang2014spectral}. They compute the second and third order empirical moments and then estimate the confusion matrices by using tensor decomposition. In particular, they randomly partition the $m$ workers into three different groups $G_1, G_2$ and $G_3$ to formulate the moments equations. For $(a, h) \in [3] \times [k]$, let $$\pi_{ah}^{\diamond}=\frac{1}{|G_a|} \sum_{i \in G_a} \pi_{h*}^{(i)}, \quad \omega_h = \frac{|\{j: y_j=h\}|}{n}.$$
Note that $\pi_{ah}^{\diamond}$ is a $k$ dimensional vector and we denote its $l$th component as $\pi_{ahl}^{\diamond}$. They use two steps to estimate the individual confusion matrices. They first estimate the aggregated confusion matrices $\pi_{a*}^{\diamond}$ by deriving equations between the moments of the labels $\{X_{ij}\}$ and the following moments of $\pi_{ah}^{\diamond}$,
$$ M_2 = \sum_{h \in [k]} \omega_h  \pi_{ah}^{\diamond} \otimes  \pi_{ah}^{\diamond}, \quad M_3= \sum_{h \in [k]} \omega_h \pi_{ah}^{\diamond}  \otimes  \pi_{ah}^{\diamond} \otimes  \pi_{ah}^{\diamond}.$$
Empirical moments are used to approximate the population moments. Due to the symmetric structure of $M_2$ and $M_3$, a robust tensor power method \cite{anandkumar2014tensor} is applied to approximately solve these equations.  Then they use another moment equation to get an estimator $\hat{\pi}^{(i)}$ of  the confusion matrices $\pi^{(i)}$ from the estimator of $\pi_{ah}^{\diamond}$.

Let $\omega_{min} = \min_{h \in [k]} \omega_h$, $\kappa=\min_{a \in [3],l \neq h \in [k]} \{\pi_{ahh}^{\diamond} - \pi_{ahl}^{\diamond}\}$ and $\sigma_k$ be the minimum $k$th eigenvalue of the matrices $S_{ab}=\sum_{h \in [k]} \omega_h \pi_{ah}^{\diamond} \otimes \pi_{bh}^{\diamond}$ for $a,b \in [3]$. Applying Theorem 1 in \cite{zhang2014spectral} to Theorem \ref{thm:adaptive}, we have the following result.

\begin{thm} \label{thm:spectral}
Assume $\log k=o(m I(\pi))$ and $\rho_m I(\pi) \le \min\{\frac{36k \kappa \log m}{\omega_{min} \sigma_L}, 2\log m\}$. Let $\hat{y}$ be the estimated labels from (\ref{eq:est}) using the estimated confusion matrices returned by Algorithm 1 in \cite{zhang2014spectral}. If the number of items $n$ satisfies
$$ n = \Omega \left( \frac{k^5 \log^3 m \log k }{\rho_m^2 I^2(\pi) \omega_{min}^2 \sigma_k^{13}} \right), $$
then for any $y \in [k]^n$, we have
$$\frac{1}{n}\sum_{j=1}^n\mathbb{I}\{\hat{y}_j\neq y_j\}\leq \exp\left(-(1+o(1))mI(\pi)\right),$$
with probability tending to $1$, where $I(\pi)$ is defined by (\ref{eq:exponent}).
\end{thm}

Combined with Theorem \ref{thm:main}, this result shows that an one-step update (\ref{eq:est}) of the spectral method proposed in \cite{zhang2014spectral} can achieve the optimal error exponent.

\subsection{One-coin Model}
For the one-coin model, a simpler method of moments for estimating $p_i$ is proposed in \cite{gao2013minimax}. Let $n_1=|\{j:y_j=1\}|$, $n_2=n-n_1$, and $\gamma=n_2/n$. They observe the equation $\frac{1}{n}\sum_{j=1}^n\mathbb{P}\left\{X_{ij}=2\right\}=\gamma p_i+(1-\gamma)(1-p_i)$. This leads to a natural estimator
\begin{equation}
\hat{p}_i=\frac{\frac{1}{n}\sum_{j=1}^n\mathbb{I}\left\{X_{ij}=2\right\}-(1-\hat{\gamma})}{2\hat{\gamma}-1}\label{eq:p-hat},
\end{equation}
where $\hat{\gamma}$ is a consistent estimator of $\gamma$ proposed 
in \cite{gao2013minimax}. Combining the consistency result of $\hat{p}_i$ 
in \cite{gao2013minimax} and Theorem \ref{thm:adaptive}, we have the following result.

\begin{thm}\label{thm: adaptonecoin}
Assume $|2\gamma-1| \ge c$ for some constant $c>0$, $\rho_m \le p_i \le 1-\rho_m$ for all $i \in [m]$ and $\frac{1}{m} \sum_{i \in [m]} (2p_i-1)^2 \le 1-\frac{4}{m}$. Let $\hat{y}$ be the estimated labels from (\ref{eq:est}) using (\ref{eq:p-hat}). If the number of items $n$ satisfies
$$ n = \Omega \left( \frac{\log^2 m \log n }{\rho_m^2 I^2(p)} \right), $$
then for any $y \in [k]^n$, we have
$$\frac{1}{n}\sum_{j=1}^n\mathbb{I}\{\hat{y}_j\neq y_j\}\leq \exp\left(-(1+o(1))mI(p)\right),$$
with probability tending to $1$, where $I(p)$ is defined by (\ref{eq:one-coin}).
\end{thm}

\section{Discussion}\label{sec:disc}

In this section, we show the implications of our results on analyzing two popular crowdsourcing algorithms, EM algorithm and  majority voting.

\subsection{EM Algorithm}
In the probabilistic model of crowdsourcing, the true labels can be regarded at latent variables. This naturally leads to apply the celebrated EM algorithm \citep{DemLaiRub77} to obtain a local optimum of maximum marginal likelihood with the following iterations \cite{DawSke79}:
\begin{itemize}
\item	(M-step) update the estimate of workers' abilities
 \begin{equation}
 \label{eq:M-step} \pi_{(t+1),gh}^{(i)} \propto \sum_{j} \mathbb{P}_{(t)}\left\{y_j=g\right\} \mathbb{I}\{X_{ij}=h\}
 \end{equation}
\item  (E-step) update the estimate of true labels
 \begin{equation}
 \label{eq:E-step} \mathbb{P}_{(t+1)}\left\{y_j=g\right\} \propto  \prod_{i, h}\left(\pi_{(t+1),gh}^{(i)}\right)^{\mathbb{I}\{X_{ij}=h\}}
 \end{equation}
\end{itemize}

The M-step (\ref{eq:M-step}) is essentially the maximum likelihood estimator.  Bayesian versions of (\ref{eq:M-step}) are considered in \cite{RayYuZha10,LiuPenIhl12}. Though the E-step (\ref{eq:E-step}) gives a probabilistic predication of the true label, a hard label can be obtained as  $\hat{y}_j=\arg\max_{g\in[k]}\mathbb{P}_{(t+1)}\left\{y_j=g\right\}$. According to Theorem \ref{thm:adaptive}, as long as the M-step gives a consistent estimate of the workers' confusion matrices, the  E-step will achieve the optimal error rate. This may explain why the EM algorithm for crowdsourcing works well in practice. In particular, as we have shown, when it is initialized by moment methods \cite{zhang2014spectral,gao2013minimax}, the EM algorithm is provably optimal after only one step of iteration.

\subsection{Majority Voting}

Majority voting is perhaps the simplest method for aggregating crowdsourced labels. In what follows,  we establish the exact error exponent of the majority voting estimator and show that it is inferior compared with the optimal error exponent. For simplicity, we only discuss the one-coin model. Then, the majority voting estimator is given by
$$\hat{y}_j=\arg\max_{g\in\{1,2\}}\sum_{i=1}^m\mathbb{I}\{X_{ij}=g\}.$$
Its error rate is characterized by the following theorem.
\begin{thm}\label{thm:MV}
Assume $ p_i\leq 1-\rho_m$ for all $i\in[m]$, $\rho_m^2 \sum_{i \in [m]} p_i(1-p_i) \to \infty$ as $m \to \infty$ and $|\log \rho_m| = o(\sqrt{m} J(p))$. Then, we have
\begin{eqnarray*}
 \sup_{y\in\{1,2\}^n}\mathbb{E}_{p,y}L(\hat{y}, y) =\exp\left(-(1+o(1))mJ(p)\right),
\end{eqnarray*}
where \[J(p)=-\min_{t\in(0,1]}\frac{1}{m}\sum_{i=1}^m\log\left[p_it+(1-p_i)t^{-1}\right].\]
\end{thm}
The theorem says that $-mJ(p)$ is the error exponent for the majority voting estimator. Given the simple relation
\begin{eqnarray}
\nonumber J(p) &=& -\min_{t\in(0,1]}\frac{1}{m}\sum_{i=1}^m\log\left[p_it+(1-p_i)t^{-1}\right] \\
\label{eq:explain} &\leq& -\frac{1}{m}\sum_{i=1}^m\min_{t>0}\log\left[p_it+(1-p_i)t^{-1}\right] \\
\nonumber &=& -\frac{1}{m}\sum_{i=1}^m\log\left(2\sqrt{p_i(1-p_i)}\right) \\
\nonumber &=& I(p),
\end{eqnarray}
we can see that the majority voting estimator has an inferior error exponent $J(p)$ to that of the optimal rate $I(p)$ in Theorem \ref{thm: adaptonecoin}.  In fact,  the inequality (\ref{eq:explain}) holds if and only if $p_i$'s are all equal, in which case, the majority voting is equivalent to the MLE (\ref{eq:MLE}). When $p_i$'s are varied among workers, majority voting cannot take the varied workers' abilities into account, thus being sub-optimal.

\section{Proofs} \label{sec:proof}
\subsection{Proof of Theorem \ref{thm:main}}
\begin{proof} 
The main proof idea is as follows. Consider the maximum likelihood estimator (\ref{eq:MLE}), we first derive the upper bound by union bound and Markov's inequality. The proof of lower bound is quite involved and it consists of three steps. Based on a standard lower bound technique, we first lower bound the misclassification rate by testing error. Then we calculate the testing error using the Neyman-Person Lemma. Finally, we give a lower bound for the tail probability of a sum of random variables, using the technique from the proof of the Cramer-Chernoff Theorem \citep[Proposition 14.23]{van2000asymptotic}.

\paragraph{Upper Bound.} Let $\hat{y}=(\hat{y}_1,...,\hat{y}_n)$ be defined as in (\ref{eq:MLE}). In the following, we  give a bound for $\mathbb{P}(\hat{y}_j\neq y_j)$. Let us denote by  $\mathbb{P}_l$ the joint probability distribution of $\{X_{ij}, i \in [m]\}$ given $\pi$ and $y_j = l$. Without loss of generality, let $y_j=1$. Using union bound, we have
$$\mathbb{P}_1(\hat{y}_j\neq 1)\leq \sum_{g=2}^k\mathbb{P}_1(\hat{y}_j= g).$$
For each $g\geq 2$, we have
\begin{eqnarray}
\nonumber \mathbb{P}_1(\hat{y}_j= g) &\leq& \mathbb{P}_1\left(\prod_{i\in[m]}\prod_{h\in[k]}(\pi_{gh}^{(i)})^{\mathbb{I}\{X_{ij}=h\}}> \prod_{i\in[m]}\prod_{h\in[k]}(\pi_{1h}^{(i)})^{\mathbb{I}\{X_{ij}=h\}}\right) \\
\nonumber &=& \mathbb{P}_1\left(\prod_{i\in[m]}\prod_{h\in[k]}\left(\frac{\pi_{gh}^{(i)}}{\pi_{1h}^{(i)}}\right)^{\mathbb{I}\{X_{ij}=h\}}>1\right) \\
\label{eq:markov} &\leq& \min_{t\geq 0}\prod_{i\in[m]}\mathbb{E}_1\prod_{h\in[k]}\left(\frac{\pi_{gh}^{(i)}}{\pi_{1h}^{(i)}}\right)^{t\mathbb{I}\{X_{ij}=h\}} \\
\nonumber &=& \min_{t\geq 0}\prod_{i\in[m]}\sum_{h\in[k]}{\left(\pi_{1h}^{(i)}\right)^{1-t}\left(\pi_{gh}^{(i)}\right)^t},
\end{eqnarray}
where (\ref{eq:markov}) is due to Markov's inequality for each $t\geq 0$. Therefore, we have
\begin{eqnarray*}
\mathbb{P}_1(\hat{y}_j\neq 1) \leq \sum_{g=2}^k\exp\left(-mC\left(\pi_{1*},\pi_{g*}\right)\right)\leq (k-1)\exp\left(-m\min_{g\neq 1}C(\pi_{1*},\pi_{g*})\right),
\end{eqnarray*}
which leads to
\begin{eqnarray*}
\frac{1}{n}\sum_{j\in[n]}\mathbb{P}_{y_j}(\hat{y}_j\neq y_j)
\leq (k-1)\exp\left(-mI(\pi)\right) =\exp\left(-(1+o(1))mI(\pi)\right),
\end{eqnarray*}
when $\log k=o(mI(\pi))$.

\paragraph{Lower Bound.} Now we establish a matching lower bound. We first introduce some notation. Define
$$B_t(\pi_{g*}^{(i)},\pi_{h*}^{(i)})=\sum_{l=1}^k\left(\pi_{gl}^{(i)}\right)^{1-t}\left(\pi_{hl}^{(i)}\right)^t.$$
Without loss of generality, we let
$$C(\pi_{1*},\pi_{2*})=\min_{g\neq h}C(\pi_{g*},\pi_{h*})=I(\pi).$$
Using the fact that the supremum over $[k]^n$ is bigger than the average over $[k]^n$ , the minimax rate $\mathcal{M}$ can be lower bounded as
\begin{eqnarray}
\nonumber \label{eq:main1} \sup_{y\in[k]^n}\mathbb{E}_{\pi,y}L(\hat{y}, y) 
\nonumber &\geq& \frac{1}{k^n}\sum_{y\in[k]^n}\mathbb{E}_{\pi,y}L(\hat{y}, y) \\
\nonumber &=& \frac{1}{kn} \sum_{l=1}^k \sum_{j=1}^n \mathbb{P}_{l}\left\{\hat{y}_j\neq l\right\}  \\
\label{eq:main1} &\geq&\frac{2}{kn}\sum_{j=1}^n\left[\frac{1}{2}\mathbb{P}_1\{\hat{y}_j\neq 1\}+\frac{1}{2}\mathbb{P}_2\{\hat{y}_j\neq 2\}\right].
\end{eqnarray}
Taking an infimum of $\hat{y}$ on both sides leads to
\begin{eqnarray} \label{eq:main2}
\inf_{\hat{y}}\sup_{y\in[k]^n}\mathbb{E}_{\pi,y}L(\hat{y}, y) 
\ge \frac{2}{kn}\sum_{j=1}^n \inf_{\hat{y}_j}\left[\frac{1}{2}\mathbb{P}_1\{\hat{y}_j=2\}+\frac{1}{2}\mathbb{P}_2\{\hat{y}_j=1\}\right].
\end{eqnarray}
By the  Neyman-Pearson Lemma \citep{neyman1933problem}, the Bayes testing error $\frac{1}{2}\mathbb{P}_1\{\hat{y}_j=2 \}+\frac{1}{2}\mathbb{P}_2\{\hat{y}_j=1\}$ is minimized by the likelihood ratio test
$$\hat{y}_j=\arg\max_{g\in\{1,2\}}\prod_{i\in[m]}\prod_{h\in[k]}\left(\pi_{gh}^{(i)}\right)^{\mathbb{I}\{X_{ij}=h\}}.$$
Therefore,
\begin{eqnarray*}
\mathbb{P}_1(\hat{y}_j=2) &=& \mathbb{P}_1\left(\prod_{i\in[m]}\prod_{h\in[k]}\left(\frac{\pi_{2h}^{(i)}}{\pi_{1h}^{(i)}}\right)^{\mathbb{I}\{X_{ij}=h\}}>1\right) = \mathbb{P}(S_m>0),
\end{eqnarray*}
where $S_m=\sum_{i\in[m]}W_i$, with the random variable $W_i$ defined as
\begin{equation} \label{eq:wi}
\mathbb{P} \left(W_i=t\log\left(\frac{\pi_{2h}^{(i)}}{\pi_{1h}^{(i)}}\right)\right)=\pi_{1h}^{(i)}.
\end{equation}
Here $t$ is a positive constant that we will specify later. We lower bound $\mathbb{P}(S_m>0)$ by
\begin{eqnarray*}
 \sum_{0<S_m}\prod_{i\in[m]}\mathbb{P}(W_i) 
&\geq& \sum_{0<S_m<L}\prod_{i\in[m]}\mathbb{P}(W_i) \\
&=& \sum_{0<S_m<L}\prod_{i\in[m]}\frac{\mathbb{P}(W_i)e^{W_i}}{B_t(\pi_{1*}^{(i)},\pi_{2*}^{(i)})}\prod_{i\in[m]}\frac{B_t(\pi_{1*}^{(i)},\pi_{2*}^{(i)})}{e^{W_i}} \\
&\geq& \prod_{i\in[m]}B_t(\pi_{1*}^{(i)},\pi_{2*}^{(i)})e^{-L}\sum_{0<S_m<L}\mathbb{Q}_i(W_i) \\
&\geq& \prod_{i\in[m]}B_t(\pi_{1*}^{(i)},\pi_{2*}^{(i)})e^{-L}\mathbb{Q}(0<S_m<L),
\end{eqnarray*}
where the distribution $\mathbb{Q}_i$ is defined as
\begin{equation} \label{eq:distQ}
\mathbb{Q}_i\left(W_i=t \log\left(\frac{\pi_{2h}^{(i)}}{\pi_{1h}^{(i)}}\right)\right)=\frac{{\left(\pi_{1h}^{(i)}\right)^{1-t}\left(\pi_{2h}^{(i)}\right)^t}}{B_t(\pi_{1*}^{(i)},\pi_{2*}^{(i)})},
\end{equation}
and $\mathbb{Q}$ is defined as the joint distribution of $\mathbb{Q}_1, \cdots, \mathbb{Q}_m$. To precede, we will need the following two lemmas. 
\begin{lemma} \label{lm:t}
If $\mathcal{A}_{\alpha}$ is not empty, there is an unique $t_0$ such that
\begin{equation} \label{eq:t0}
t_0 = \argmin_{t \in [0,1]} \prod_{i\in[m]}B_t(\pi_{1*}^{(i)},\pi_{2*}^{(i)}). 
\end{equation}
Moreover, we have $0 < t_0 <1$.
\end{lemma}
\begin{lemma} \label{lm:clt}
Let $t=t_0$ defined in (\ref{eq:t0}). Then under the assumption of Theorem \ref{thm:main}, $S_m$ is a zero mean random variable satisfying the central limit theorem, i.e. for any x,
$$ \mathbb{Q}\left( \frac{S_m}{\sqrt{\text{Var}(S_m)}} \le x  \right)  \to \Phi(x), \text{~as~} m \to \infty,$$
where $\Phi$ is the cumulative distribution function of a $N(0,1)$ random variable.
\end{lemma}

The proof of Lemma \ref{lm:t} and Lemma \ref{lm:clt} are deferred to Section \ref{sec:tech}. Let $t=t_0$ and $L=2\sqrt{\text{Var}_Q(S_m)}.$   Using Lemma \ref{lm:clt} and Chebyshev's inequality, we have
\begin{eqnarray*}
\mathbb{Q}(0<S_m<L) \ge 1 -  \mathbb{Q}(S_m \le 0) - \mathbb{Q}(S_m \ge L) 
\ge 1 - 5/8 - 1/4 = 1/8
\end{eqnarray*}
for sufficiently large $m$. Note that
\begin{eqnarray*}
 \mathbb{E}_QW_i^2
&=& \sum_{h\in[k]}\left(t\log\left(\frac{\pi_{2h}^{(i)}}{\pi_{1h}^{(i)}}\right)\right)^2\mathbb{Q}_i\left(W_i=t\log\left(\frac{\pi_{2h}^{(i)}}{\pi_{1h}^{(i)}}\right)\right) \\
&\leq& \max_{i,h} \left(t\log\left(\frac{\pi_{2h}^{(i)}}{\pi_{1h}^{(i)}}\right)\right)^2 \sum_{h\in[k]} \mathbb{Q}_i\left(W_i=t\log\left(\frac{\pi_{2h}^{(i)}}{\pi_{1h}^{(i)}}\right)\right) \\
&\leq& \log^2 \rho_m.
\end{eqnarray*}
Consequently,
\begin{eqnarray*}
\text{Var}_\mathbb{Q}(S_m)
 = \sum_{i\in[m]}\text{Var}_\mathbb{Q}(W_i)
\leq \sum_{i\in[m]}\mathbb{E}_QW_i^2 
\leq m \log^2 \rho_m.
\end{eqnarray*}
Under the assumption that $\log^2 \rho_m = o(m I^2(\pi))$, we have $e^{-L} \geq e^{-\sqrt{m \log^2 \rho_m}} \ge e^{-o(mI(\pi))}$. This leads to the lower bound
\begin{eqnarray*}
\mathbb{P}_1(\hat{y}_j=2) \geq \prod_{i\in[m]}B_t(\pi_{1*}^{(i)},\pi_{2*}^{(i)})e^{-o\left(mI(\pi)\right)}
=\exp\left(-(1+o(1))mI(\pi)\right).
\end{eqnarray*}
Note that the same bound holds for $\mathbb{P}_2(\hat{y}_j = 1)$. Hence,
\begin{eqnarray*}
\inf_{\hat{y}}\sup_{y\in[k]^n}\mathbb{E}L(\hat{y}, y)
\geq \frac{2}{k}\exp\left(-(1+o(1))mI(\pi)\right)
=\exp\left(-(1+o(1))mI(\pi)\right),
\end{eqnarray*}
under the assumption that $\log k=o\left(mI(\pi)\right)$. This completes the proof.
\end{proof}

\subsection{Proof of Corollary \ref{cor:one-coin}}
\begin{proof}
Under the assumption that $mI(\pi)\rightarrow\infty$, the upper bound is a special case of Theorem \ref{thm:main}. Note that
\begin{eqnarray*} 
I(\pi) &=&-\min_{0\leq t\leq 1}\frac{1}{m}\sum_{i=1}^m\log\left(p_i^{1-t}(1-p_i)^t+p_i^{1-t}(1-p_i)^t\right)\\
&=&-\frac{1}{m}\sum_{i=1}^m\log\left(2\sqrt{p_i(1-p_i)}\right)\\
&=& I(p).
\end{eqnarray*}
We focus on the proof of the lower bound, which involves weaker assumptions than that of Theorem \ref{thm:main}. Using a similar analysis as (\ref{eq:main1})-(\ref{eq:main2}), we have
\begin{eqnarray*}
 \inf_{\hat{y}}\sup_{y\in\{1,2\}^n}\mathbb{E}L(\hat{y}, y)
\geq \frac{1}{n}\sum_{j=1}^n\inf_{\hat{y}_j}\left[\frac{1}{2}\mathbb{P}_1\{\hat{y}_j =2\}+\frac{1}{2}\mathbb{P}_2\{\hat{y}_j= 1\}\right].
\end{eqnarray*}
Following the proof of Theorem \ref{thm:main} with the confusion matrix $\pi^{(i)}$ replaced by (\ref{eq:special}), we have
\begin{eqnarray*}
\inf_{\hat{y}_j}\left[\frac{1}{2}\mathbb{P}_1\{\hat{y}_j=2\}+\frac{1}{2}\mathbb{P}_2\{\hat{y}_j=1\}\right]
\geq \exp\left(-mI(p)\right)e^{-L}\mathbb{Q}(0<S_m<L),
\end{eqnarray*}
where $S_m=\sum_{i\in[m]}W_i$, and under the distribution $\mathbb{Q}$,
$$\mathbb{Q}_i\left(W_i=\frac{1}{2}\log\frac{1-p_i}{p_i}\right)=\mathbb{Q}_i\left(W_i=\frac{1}{2}\log\frac{p_i}{1-p_i}\right)=\frac{1}{2}.$$
Therefore, $S_m$ has a symmetric distribution around $0$. Letting $L=2\sqrt{\text{Var}_Q(S_m)}$, we have
$$\mathbb{Q}(0<S_m<L)\geq \frac{1}{2}-\mathbb{Q}(S_m\geq L)\geq \frac{1}{2}-\frac{\text{Var}_Q(S_m)}{L^2}\geq \frac{1}{4}.$$
Finally, we need to show that $L=o(mI(p))$. We claim that
\begin{eqnarray*}
\sum_{i=1}^m\text{Var}_QW_i &=&\frac{1}{4}\sum_{i=1}^m\left(\log\frac{1-p_i}{p_i}\right)^2 \\
&\le& -8\max_{1\leq i\leq m}(|\log(p_i)|\vee|\log(1-p_i)| \vee 2)  \sum_{i=1}^m\log\left(2\sqrt{p_i(1-p_i)}\right).
\end{eqnarray*}
This is becase when $p_i \in [1/16, 15/16]$, we have $\left | \log\frac{1-p_i}{p_i}\right |^2 \le  6(2p_i-1)^2\le -6\log\left(4p_i(1-p_i)\right)$. When $p_i \in (0,1/16) \cup (15/16,1)$, $\left|\log\frac{1-p_i}{p_i}\right| \le -2 \log \left(4p_i(1-p_i)\right)$ and $\left|\log\frac{1-p_i}{p_i}\right| \le 2 |\log(p_i)| \vee 2|\log(1-p_i)|$. Therefore, under the assumption that
$$\max_{1\leq i\leq m}(|\log(p_i)|\vee|\log(1-p_i)|)=o(mI(p)),$$
$L=o(mI(p))$ holds, and the proof is complete.
\end{proof}

\subsection{Proof of Theorem \ref{thm:adaptive}}
\begin{proof}
Define
$$E=\left\{\max_{g\in[k]}\sum_{i\in[m]} \max_{h \in [k]} \left|\log\hat{\pi}_{gh}^{(i)}-\log\pi_{gh}^{(i)}\right|\leq \delta\right\}.$$
Then, we have
\begin{eqnarray*}
\mathbb{P}\left(\frac{1}{n}\sum_{j}\mathbb{I}\{\hat{y}_j\neq y_j\}>\epsilon\right) 
&\leq& \mathbb{P}\left(\frac{1}{n}\sum_{j}\mathbb{I}\{\hat{y}_j\neq y_j\}>\epsilon, E\right) + \mathbb{P}(E^c) \\
&=& \mathbb{P}\left(\frac{1}{n}\sum_{j}\mathbb{I}\{\hat{y}_j\neq y_j\}>\epsilon\Big| E\right)\mathbb{P}(E) + \mathbb{P}(E^c) \\
&\leq& \frac{1}{n}\sum_{j}\mathbb{P}\left(\hat{y}_j\neq y_j |E\right)\mathbb{P}(E)/\epsilon + \mathbb{P}(E^c) \\
&=& \frac{1}{n}\sum_{j}\mathbb{P}\left(\hat{y}_j\neq y_j, E\right)/\epsilon + \mathbb{P}(E^c).
\end{eqnarray*}
Let us give a bound for $\mathbb{P}\left(\hat{y}_j\neq y_j, E\right)$. Without loss of generality, let $y_j=1$. Then,
\begin{eqnarray*}
\mathbb{P}\left(\hat{y}_j\neq y_j, E\right) 
&\leq& \sum_{g=2}^k\mathbb{P}\left(\hat{y}_j=g, E\right) \\
&\leq& \sum_{g=2}^k\mathbb{P}\left(\prod_{i\in[m]}\prod_{h\in[k]}\left(\frac{\hat{\pi}_{gh}^{(i)}}{\hat{\pi}_{1h}^{(i)}}\right)^{\mathbb{I}\{X_{ij}=h\}}>1, E\right) \\
&=& \sum_{g=2}^k\mathbb{P}\left(\prod_{i\in[m]}\prod_{h\in[k]}\left(\frac{{\pi}_{gh}^{(i)}}{{\pi}_{1h}^{(i)}}\right)^{\mathbb{I}\{X_{ij}=h\}} \prod_{i\in[m]}\prod_{h\in[k]}\left(\frac{\hat{\pi}_{gh}^{(i)}\pi_{1h}^{(i)}}{\pi_{gh}^{(i)}\hat{\pi}_{1h}^{(i)}}\right)^{\mathbb{I}\{X_{ij}=h\}}>1, E\right).
\end{eqnarray*}
On the event $E$,
\begin{eqnarray*} 
\log \left( \prod_{i\in[m]}\prod_{h\in[k]}\left(\frac{\hat{\pi}_{gh}^{(i)}\pi_{1h}^{(i)}}{\pi_{gh}^{(i)}\hat{\pi}_{1h}^{(i)}}\right)^{\mathbb{I}\{X_{ij}=h\}} \right)
 \le \sum_{i \in [m]} \sum_{h \in [k]} \left( \log \frac{\hat{\pi}_{gh}^{(i)}}{\pi_{gh}^{(i)}} - \log \frac{\hat{\pi}_{1h}^{(i)}}{\pi_{1h}^{(i)}} \right)\mathbb{I}\{X_{ij}=h\} \le 2\delta.
 \end{eqnarray*}
Then
\begin{eqnarray*}
\mathbb{P}\left(\hat{y}_j\neq y_j, E\right) 
&\leq& \sum_{g=2}^k\mathbb{P}\left(e^{2\delta}\prod_{i\in[m]}\prod_{h\in[k]}\left(\frac{{\pi}_{gh}^{(i)}}{{\pi}_{1h}^{(i)}}\right)^{\mathbb{I}\{X_{ij}=h\}}>1\right) \\
&\leq& \sum_{g=2}^ke^{2\delta}\min_{0\leq t\leq 1}\prod_{i\in[m]}\sum_{h\in[k]}{\left(\pi_{1h}^{(i)}\right)^{1-t}\left(\pi_{gh}^{(i)}\right)^t} \\
&\leq& (k-1)\exp\left(-m\min_{g\neq 1}C(\pi_{1*},\pi_{g*})+2\delta\right).
\end{eqnarray*}
Thus,
$$\frac{1}{n}\sum_{j\in[n]}\mathbb{P}\left(\hat{y}_j\neq y_j, E\right)\leq (k-1)\exp\left(-m I(\pi)+2\delta\right).$$
Letting $\epsilon=(k-1)\exp\left(-(1-\eta)mI(\pi)+2\delta\right)$ with $\eta=1/\sqrt{mI(\pi)}$, we have
$$\frac{1}{n}\sum_{j}\mathbb{P}\left(\hat{y}_j\neq y_j, E\right)/\epsilon\leq \exp\left(-\sqrt{mI(\pi)}\right).$$
Thus, the proof is complete under the assumption that $\log k+\delta=o(mI(\pi))$ and $\mathbb{P}(E^c)=o(1)$.
\end{proof}

\subsection{Proof of Theorem \ref{thm:MV}}
\begin{proof}
The risk is $\frac{1}{n}\sum_{j=1}^n\mathbb{P}\{\hat{y}_j\neq y_j\}$. Consider the random variable $\mathbb{I}\{\hat{y}_j\neq y_j\}$. It has the same distribution as $\mathbb{I}\{\sum_{i=1}^m(T_i-1/2)>0\}$, where $T_i\sim\text{Bernoulli}(1-p_i)$. Therefore,
\begin{eqnarray*}
&& \frac{1}{n}\sum_{j=1}^n\mathbb{P}\{\hat{y}_j\neq y_j\}
=\mathbb{P}\left\{\sum_{i=1}^m(T_i-1/2)>0\right\}.
\end{eqnarray*}

We first derive the upper bound. Using Chernoff's method, we have
\begin{eqnarray*}
\mathbb{P}\left\{\sum_{i=1}^m(T_i-1/2)>0\right\}
\leq \prod_{i=1}^m\mathbb{E}e^{\lambda(T_i-1/2)} 
= \exp\left(\sum_{i=1}^m\log\left[(1-p_i)e^{\lambda/2}+p_ie^{-\lambda/2}\right]\right).
\end{eqnarray*}
The desired upper bound follows by letting $t=e^{-\lambda/2}$ and optimizing over $t\in(0,1]$.

Now we show the lower bound using the similar arguments as in the proof of Theorem \ref{thm:main}. Define $W_i=\lambda (T_i-1/2)$ and $S_m=\sum_{i=1}^m W_i$. Then, we have
\begin{eqnarray*}
\mathbb{P}\left\{\sum_{i=1}^m(T_i-1/2)>0\right\}
&=& \mathbb{P}\left\{S_m>0\right\} \\
&\geq& \sum_{0<S_m<L}\prod_{i=1}^m \mathbb{P}(W_i) \\
&=& \sum_{0<S_m<L}\left(\prod_{i=1}^m\frac{\mathbb{P}(W_i)e^{W_i}}{(1-p_i)e^{\lambda/2}+p_ie^{-\lambda/2}}\prod_{i=1}^m\frac{(1-p_i)e^{\lambda/2}+p_ie^{-\lambda/2}}{e^{W_i}}\right) \\
&\geq& \prod_{i=1}^m\left((1-p_i)e^{\lambda/2}+p_ie^{-\lambda/2}\right)e^{-L}\mathbb{Q}\left\{0<S_m<L\right\}.
\end{eqnarray*}
Note that under $\mathbb{Q}$, $W_i$ has distribution
\begin{eqnarray*}
Q_i(W_i=\lambda/2)&=&\frac{(1-p_i)e^{\lambda/2}}{(1-p_i)e^{\lambda/2}+p_ie^{-\lambda/2}},\\ Q_i(W_i=-\lambda/2)&=& \frac{p_ie^{-\lambda/2}}{(1-p_i)e^{\lambda/2}+p_ie^{-\lambda/2}}.
\end{eqnarray*} We choose $\lambda_0 \in [0, \infty)$ to minimize $f(\lambda)=\prod_{i=1}^m\left((1-p_i)e^{\lambda/2}+p_ie^{-\lambda/2}\right)$. This leads to the equation $\mathbb{E}_QS_m=0$. It is sufficient to lower bound $e^{-L}\mathbb{Q}\left\{0<S_m<L\right\}$ to finish the proof. To do this, we need the following result.
\begin{lemma}\label{lm:miniclt}
Suppose $p_i \le 1 - \rho_m$ for all $i \in [m]$ and $\rho_m^2 \sum_{i \in [m]} p_i(1-p_i) \to \infty$ as $m \to \infty$. Then we have
\begin{itemize}
\item[i)] $\lambda_0 \le -2 \log \rho_m$.
\item[ii)] $\frac{S_m}{\sqrt{\text{Var}_Q(S_m)}}\leadsto N(0,1)$ under the distribution $\mathbb{Q}$.
\end{itemize}
\end{lemma}
The proof of Lemma \ref{lm:miniclt} will be given in Section \ref{sec:tech}.
Let $L=2\sqrt{\text{Var}_Q(S_m)}$, and we have $$e^{-L}\mathbb{Q}(0<S_m<L)\ge 0.25 e^{-2\sqrt{\text{Var}_Q(S_m)}}.$$ Finally, we need to show $\sqrt{\text{Var}_Q(S_m)}=o(mJ(p))$. This is because 
\begin{eqnarray*}
\text{Var}(S_m) \leq \sum_{i=1}^m\mathbb{E}_QW_i^2\leq m \lambda_0^2/4 
\le m  \log^2{\rho_m}=o(m^2 J(p)^2),
\end{eqnarray*}
where the last equality is implied by the assumption $|\log \rho_m| = o(\sqrt{m} J(p))$. The proof is complete.
\end{proof}

\subsection{Proof of Technical Lemmas} \label{sec:tech}
\begin{proof}[Proof of Lemma \ref{lm:t}]
Let $f(t) = \sum_{i=1}^{m}\log B_t(\pi_{1*}^{(i)},\pi_{2*}^{(i)})$. Then we have $f'(t_0)=0$ by its definition. First, we are going to prove $0 < t_0 <1$. The concavity of logarithm gives us $x^t y^{1-t} \le t x + (1-t) y$ for non-negative $x,y$ and $t \in [0,1]$, which implies
$$f(t) = \sum_{i \in [m]} \log B_t (\pi_{1*}^{(i)},\pi_{2*}^{(i)}) \le \sum_{i \in [m]} \log \left( \sum_{h=1}^{k} \left( (1-t) \pi_{1h}^{(i)} + t \pi_{2h}^{(i)} \right) \right) = 0.$$
For $t \in (0,1)$, the equality holds if and only if $\pi_{1h}^{(i)}=\pi_{2h}^{(i)}$ for all $h \in [k]$ and $i \in [m]$. As there is at least one non-spammer, we must have $f(t) < 0=f(0)=f(1)$ for $t \in (0,1)$.  Hence the minimizer $t_0 \in (0,1)$.

Now we are going to show the uniqueness of $t_0$ by proving that 
$$ f''(t) = \text{Var}(S_m) >0, ~\forall t \in (0,1)$$
where $S_m=\sum_{i \in [m]} W_i$. To simplify the notation, let us define $w_{ih} = t \log\left(\frac{\pi_{2h}^{(i)}}{\pi_{1h}^{(i)}}\right)$ and $p_{ih}={\left(\pi_{1h}^{(i)}\right)^{1-t}\left(\pi_{2h}^{(i)}\right)^t}$ for all $i \in [m]$ and $h \in [k]$. Now
$ \mathbb{Q}_i (W_i=w_{ih}) = p_{ih}/\sum_{h} p_{ih} $ and $B_t(\pi_{1*}^{(i)},\pi_{2*}^{(i)}) = \sum_{h \in [k]} p_{ih}$.  Notice that $\frac{d}{dt} p_{ih} = p_{ih} w_{ih},$ we have
\begin{equation} \label{eq:mean} 
\frac{d}{dt} f(t) = \sum_{i \in [m]} \frac{\sum_{h}p_{ih} w_{ih}}{\sum_{h} p_{ih}} = \sum_{i \in [m]} \mathbb{E} W_i = \mathbb{E} S_m,
\end{equation}
and
\begin{equation}  \label{eq:variance}
\frac{d^2}{dt^2} f(t) =  \sum_{i \in [m]} \frac{\sum_{h}p_{ih} w^2_{ih} \sum_{h} p_{ih} -  \left(\sum_{h}p_{ih} w_{ih}\right)^2 }{\left(\sum_{h} p_{ih}\right)^2} = \sum_{i \in [m]} \text{Var}(W_i) = \text{Var}(S_m). 
\end{equation}
Since the set $\mathcal{A}_{\alpha}$ is non-empty, there is at least one $\text{Var}(W_i)>0$. Thus, $f''(t)=\text{Var}(S_m)>0$. 
\end{proof}

\begin{proof}[Proof of Lemma \ref{lm:clt}]
From (\ref{eq:mean}), we know $\mathbb{E} S_m = f'(t_0)=0$. Since $t_0>0$ by lemma 6.1, we can rescale $W_i$ by $W_i/(-t_0\log \rho_m)$ and the value of $S_m/\sqrt{\text{Var}(S_m)}$ will not change. Let us define $V_i=W_i/(-t_0\log \rho_m)$ and $R_m=\sum_{i=1}^{m} V_i$. Then we have $|V_i| \le 1$. To prove a central limit theorem of $S_m$, it is sufficient to check the following Lindeberg's condition \cite{durrett2010probability}, that is,  for any $\epsilon>0$, 
\begin{eqnarray} \label{eq:linderberg}
\frac{1}{\text{Var}(R_m)} \sum_{i=1}^{m} \mathbb{E} \left( (V_i - \mathbb{E} V_i)^2 \mathbf{I}\{(V_i - \mathbb{E} V_i)^2 \ge \epsilon^2 \text{Var}(R_m)\} \right) \to 0 \text{~as~} m \to \infty.
\end{eqnarray}
 Note that for a discrete random variable $X$ who takes value $x_a$ with probability $p_a$ for $a \in [N]$, 
$$ \text{Var} (X) = \left( \sum_a p_a \right) \left( \sum_{a} p_a x_a^2 \right) - \left (\sum_{a} p_a x_a \right)^2 = \sum_{a,b} p_a p_b (x_a - x_b)^2.$$
Then, for any $i \in \mathcal{A}_\alpha$, we have
\begin{eqnarray*} 
\text{Var} (V_i) &= & \frac{1}{\log^2 \rho_m} \sum_{a,b} \frac{{\left(\pi_{1a}^{(i)}\pi_{1b}^{(i)}\right)^{1-t}\left(\pi_{2a}^{(i)} \pi_{2b}^{(i)} \right)^t}}{B_t^2(\pi_{1*}^{(i)},\pi_{2*}^{(i)})} \log^2 \left(\frac{\pi_{2a}^{(i)}\pi_{1b}^{(i)}}{\pi_{1a}^{(i)}\pi_{2b}^{(i)}}\right) \\
&\ge& \frac{1}{\log^2 \rho_m} {\left(\pi_{12}^{(i)}\pi_{11}^{(i)}\right)^{1-t}\left(\pi_{22}^{(i)} \pi_{21}^{(i)} \right)^t} \log^2 \left(\frac{\pi_{22}^{(i)}\pi_{11}^{(i)}}{\pi_{12}^{(i)}\pi_{21}^{(i)}}\right)  \\
&\ge& \frac{1}{\log^2 \rho_m}  {\left(\pi_{12}^{(i)}\pi_{11}^{(i)}\right)^{1-t}\left(\pi_{22}^{(i)} \pi_{21}^{(i)} \right)^t} \log^2 \left( (1+\alpha)^2 \right)\\
&\ge& \frac{\rho_m^2}{\log^2 \rho_m}  4 \log^2(1+\alpha) \\
&\ge& \frac{\rho_m^2}{\log^2 \rho_m}   \min\{\alpha^2,1\}.
 \end{eqnarray*}
Here the second inequality is due to the assumption that for any $i \in \mathcal{A}$, $\pi_{aa}^{(i)} \ge \pi_{ab}^{(i)} (1+\alpha)$ for any $b \neq a$. We have used the assumption that $\pi_{ab}^{(i)} \ge \rho_m$ for the third inequality. The last inequality is because $\log(1+\alpha) \ge \alpha/(1+\alpha) \ge \min\{\alpha/2,1/2\}$ for positive $\alpha$. Take a sum of $ \text{Var}(V_i)$ over $i \in \mathcal{A}_\alpha$,
$$ \text{Var}(R_m) =  \sum_{i \in [m]} \text{Var}(V_i) \ge |\mathcal{A}_\alpha|  \frac{\rho_m^2}{\log^2 \rho_m} \min\{\alpha^2,1\} \ge c m  \frac{\rho_m^2}{\log^2 \rho_m} \min\{\alpha^2,1\},$$
for some constant $c \in (0,1)$. Since $(V_i - \mathbb{E} V_i)^2 \le 2V_i^2 + 2  (\mathbb{E} V_i)^2  \le 4$, we will have
$$\mathbf{I}\{(V_i - \mathbb{E} V_i)^2 \ge \epsilon^2 \text{Var}(R_m)\} = 0$$ 
when $4 \log^2 \rho_m < \epsilon^2 |\mathcal{A}_\alpha| \rho_m^2 \min\{\alpha^2,1\}$. Notice that $\mathbb{E} (V_i - \mathbb{E} V_i)^2 \le 4$, we apply the Dominated Convergence Theorem to conclude $$ \mathbb{E} \left( (V_i - \mathbb{E} V_i)^2 \mathbf{I}\{(V_i - \mathbb{E} V_i)^2 \ge \epsilon^2 \text{Var}(R_m)\} \right) \to 0.$$   Thus, the Lindeberg condition holds when $| \log \rho_m | =o( \rho_m |\mathcal{A}_{0.01}|^{1/2} )$. 
\end{proof}

\begin{proof}[Proof of Lemma \ref{lm:miniclt}]
We are first going to show $\lambda_0 \le -2\log \rho_m$. Recall that $$f(\lambda)=\prod_{i=1}^m\left((1-p_i)e^{\lambda/2}+p_ie^{-\lambda/2}\right).$$For all $\lambda\geq  -2\log \rho_m$, we have $f(\lambda) > \rho_m^{-m} \prod_{i=1}^m(1-p_i)\geq 1$, and $f(0)=1$. Thus, the minimizer of $f(\lambda)$ must be in the interval $(0,-2\log \rho_m]$.

Again, we are going to prove the central limit theorem of $S_m$ by checking the following Lindeberg's condition. For any $\epsilon>0$,
\begin{equation}
\lim_{m\rightarrow\infty} \frac{1}{\text{Var}_\mathbb{Q}(S_m)}\sum_{i=1}^m\mathbb{E}_\mathbb{Q}\left[(W_i-\mathbb{E}_\mathbb{Q}W_i)^2\mathbf{I}\left\{|W_i-\mathbb{E}_\mathbb{Q}W_i|>\epsilon\sqrt{\text{Var}_\mathbb{Q}(S_m)}\right\}\right]=0\label{eq:LINDE}
\end{equation}
When $\lambda_0 \in(0,-2\log \rho_m]$, a lower bound of $\text{Var}_Q(W_i)$ is given by
\begin{eqnarray*}
\text{Var}_\mathbb{Q}(W_i) &=& \lambda_0^2\frac{p_i(1-p_i)}{\left((1-p_i)e^{\lambda_0/2}+p_ie^{-\lambda_0/2}\right)^2} 
\geq \lambda_0^2 e^{-\lambda} p_i(1-p_i) 
\geq \lambda_0^2 \rho_m^2 p_i(1-p_i)
\end{eqnarray*}
Therefore, $\text{Var}_\mathbb{Q}(S_m)=\sum_{i=1}^m\text{Var}_\mathbb{Q}(W_i)\geq \lambda_0^2 \rho_m^2 \sum_{i \in [m]}p_i(1-p_i)$. Notice that $|W_i - \mathbb{E}_\mathbb{Q} W_i| \le |W_i| + \mathbb{E} |W_i| = \lambda$,  for any fixed $\epsilon>0$, we will have 
$$\mathbf{I}\left\{|W_i-\mathbb{E}_\mathbb{Q}W_i|>\epsilon\sqrt{\text{Var}_\mathbb{Q}(S_m)}\right\} = 0$$ 
when   $\rho_m^2 \sum_{i \in [m]}p_i(1-p_i) \to \infty$ as $m \to \infty$. Since $\text{Var}_\mathbb{Q}(W_i)/\lambda_0^2 \le 1/4$, the Dominated Convergence Theorem implies the desired Lindeberg's condition (\ref{eq:LINDE}). 
\end{proof}

\bibliography{ref}
\bibliographystyle{plain}

\end{document}